\documentclass{article}
\usepackage{microtype}
\usepackage{graphicx}

\usepackage{mathtools} 
\usepackage{dsfont} 
\usepackage{amssymb} 
\usepackage{amsfonts} 
\usepackage{mathrsfs} 
\usepackage{amsthm} 
\usepackage{bm} 
\usepackage{booktabs} 
\usepackage{tabularx} 
\usepackage{float} 
\usepackage{multirow} 
\usepackage{mymath} 
\usepackage{tikz-setup} 
\usepackage{cancel} 
\usepackage{balance}
\usepackage{float}
\usepackage[caption = false]{subfig}
\newcommand{\SW}{\texttt{SEW-IL}}

\title{Semi-supervised Wrapper Feature Selection by Modeling Imperfect Labels}
\author{
Vasilii Feofanov, Emilie Devijver, Massih-Reza Amini\\
Université Gwrenoble Alpes, CNRS\\
LIG/IMAG, 700 av. centrale 
38401 Saint-Martin d'Hères, France \\
\url{FirstName.Lastname@univ-grenoble-alpes.fr}
}
\date{}
\begin{document}

\maketitle

\begin{abstract}
In this paper, we propose a new wrapper feature selection approach with partially labeled training examples where unlabeled observations are  pseudo-labeled using the predictions of an initial classifier trained on the labeled training set. The wrapper is composed of a genetic algorithm for proposing new feature subsets, and an evaluation measure for scoring the different feature subsets. The selection of feature subsets is done by assigning weights to characteristics and recursively eliminating those that are irrelevant. The selection criterion is based on a new multi-class $\mathcal{C}$-bound that explicitly takes into account the mislabeling errors induced by the pseudo-labeling mechanism, using a probabilistic error model. Empirical results on different data sets show the effectiveness of our framework compared to several state-of-the-art semi-supervised feature selection approaches.
\end{abstract}

\section{Introduction}
We consider semi-supervised learning problems where observations are described by a large number of characteristics. In this case, the original set of features may contain \emph{irrelevant} or \emph{redundant} characteristics to the output, which with the lack of labeled information, lead to inefficient learning models. In practice, the removal of such features has been shown to provide important keys for the interpretability of results and yield to better model predictions \cite{Guyon:2003,Chandrashekar:2014}.

Depending on the availability of class labels, feature selection techniques can be supervised, unsupervised or semi-supervised. Being agnostic to the target variable, unsupervised approaches generally ignore the discriminative power of features, so their use may lead to poor performance. In contrast, supervised feature selection algorithms benefit from abundant labeled examples, so they effectively select the subset of relevant characteristics.
In semi-supervised feature selection \cite{Sheikhpour:2017}, the aim is to exploit both available labeled and unlabeled training observations in order to provide a solution that preserves important structures of data and leads to high performance.  Considerable progress has been made in this direction over the last few years. \emph{Filter methods}  \cite{Yang:2010,Zhao:2008} score features following a criterion and perform selection before a learning model is constructed. \emph{Embedded techniques} \cite{Chen:2017} perform model-based feature selection in order to infer the importance of features during the training process. Finally, \emph{Wrapper approaches} \cite{Kohavi:1997,Ren:2008} use a learner to effectively find a subset of features that are discriminatively powerful together. The underlying principle of these approaches is to search a feature subset by optimizing the prediction performance of a given learning model. However, the bottleneck is the optimization of these approaches as well as the use of information provided by unlabeled training samples. One way to tackle the latter is to increase data diversity, by pseudo-labeling unlabeled examples using either self-learning or co-training approaches \cite{Blum:1998,Tur:2005,Amini:2009,Vittaut02}. In this case, pseudo-labels are iteratively assigned to unlabeled examples with confidence score above a certain threshold. However, fixing this threshold is in some extent tricky. In addition, the pseudo-labels may be prone to error and their use become therefore more difficult and hazardous. On the other hand, using exhaustive search for its optimization would be computationally infeasible due to an exponential number of possible subsets. In addition, the use of sequential search algorithms like the one proposed by \cite{Ren:2008}, in the case of very large dimension, becomes also infeasible. To overcome this problem, a common practice is to use heuristic search algorithms, for instance, a genetic algorithm \cite{Goldberg:1988}. However, for applications of large dimension, this approach may have large variance in output, and the set of selected features might be still large.

In this paper, we propose a new framework for semi-supervised wrapper feature selection, referred to as \SW{}. The approach is based on data-augmentation obtained by a self-learning algorithm applied to unlabeled training samples \cite{Feofanov:2019}. The mislabeling errors of pseudo-labels assigned to unlabeled training data are estimated using a probabilistic model. Based on these estimates, we derive a new upper bound of the Bayes classifier that extends the $\mathcal{C}$-bound, proposed by \cite{Lacasse:2007} for the Bayes risk in the supervised case. Our bound is obtained by considering the mean and the variance of the margin predictions of the classifier over unlabeled samples.  Finally, we propose a modification of a genetic algorithm that uses this bound as a selection criterion by taking into account feature weights during the optimization phase.

In the reminder of the paper, we introduce in Section \ref{sec:framework} the theoretical framework and derive a new $\mathcal{C}$-bound with partially labeled data by modeling the mislabeling errors. In Section \ref{sec:algorithm}, we describe our semi-supervised method to select relevant features and we present the experimental results we obtained on ten benchmarks in Section \ref{sec:exp}. Finally, we discuss the outcomes of this study in Section \ref{sec:conclusion}.

\section{Probabilistic $\mathcal{C}$-Bound with Imperfect Labels}
\label{sec:framework}
\subsection{Framework}

We consider multi-class classification problems with an input space $\mathcal{X}\subset \R^d$ and an output space $Y~\in~\mathcal{Y}~=~\{1,\dots,K\}$, $K\geq 2$. We denote by $\mbf{X}=(X_1,\ldots,X_d)\in\mathcal{X}$ (resp. $Y\in\mathcal{Y}$)  an input (resp. output) random variable and $\mbf{X}_{[S]}$ the input projection on a subset of features $S\subset D=\{1,\dots,d\}$. As we are dealing with semi-supervised learning, we assume available a set of labeled training examples $\mathrm{Z}_{\mathcal{L}}=\{(\mathbf{x}_i,y_i)\}_{i=1}^l\in(\mathcal{X}\times\mathcal{Y})^l$, identically and independently distributed  (i.i.d.) with respect to a fixed yet unknown probability distribution $P(\mbf{X}, Y)$ over $\mathcal{X}\times\mathcal{Y}$, and a set of unlabeled training examples $\mathrm{X}_{\sss\mathcal{U}} = \{\mathbf{x}'_i\}_{i=l+1}^{l+u}\in \mathcal{X}^u$ supposed to be drawn i.i.d. from the marginal distribution $P(\mathbf{X})$, over the domain $\mathcal{X}$. 

Following \cite{Koller:1996}, we call $\mbf{X}_{[S]}$  a Markov blanket for the output $Y$ if $Y\perp \mbf{X}_{[D\setminus S]}|\mbf{X}_{[S]}$, and formulate the goal of semi-supervised feature selection as to find a minimal Markov blanket among all possible $2^d$ feature subsets based on the available labeled and unlabeled data. 
By searching a solution that is Markov blanket we exclude all irrelevant to $Y$ variables. The minimal property indicates the removal of redundant variables.

In this work, a fixed class of classifiers $\mathcal{H}~=~\{h~|~h~:~\mathcal{X}~\rightarrow~\mathcal{Y}\}$, called the \emph{hypothesis space}, is considered and defined without reference to the training set. Further, we focus on the \emph{Bayes classifier} (also called the majority vote classifier) defined as~:
\begin{equation}
\label{eq:Bayes-classifier-multi}
\forall \mathbf{x}\in \mathcal{X}; B_Q(\mathbf{x}):= \argmax_{c\in\{1,\ldots,K\}}\left[\E_{h\sim Q}\I{h(\mathbf{x})=c}\right].
\end{equation}

We formulate the task of the learner as to choose a posterior distribution $Q$ over $\mathcal{H}$ observing the training set $\mathrm{Z}_{\mathcal{L}}\cup~\mathrm{X}_{\sss\mathcal{U}}$ such that the true risk of the classifier is minimized:
\begin{equation*}
R(B_Q) := \E_{P(\mbf{X},Y)}\I{B_Q(\mathbf{X})\neq Y}.
\end{equation*}

We are increasing the labeled set by pseudo-labeling observations in which we have a strong confidence for the prediction. It is measured through the \emph{margin}, defined in the following way for  an observation $\mathbf{x}$:
\begin{align*}
M_Q(\mathbf{x}, y) &:= \E_{h\sim Q}\I{h(\mathbf{x})=y} - \max_{\substack{{c\in\mathcal{Y}}\\{c\neq y}}} \E_{h\sim Q}\I{h(\mathbf{x})=c}\\
&= v_Q(\mbf{x},y) - \max_{\substack{{c\in\mathcal{Y}}\\{c\neq y}}} v_Q(\mbf{x},c),    
\end{align*}
where $v_Q(\mbf{x}, c)$ is the \emph{vote} given by the Bayes classifier $B_Q$ to the class membership of an example $\mbf{x}$ being $c$.
If it is strictly positive for an example $\mbf{x}$, then the example is correctly classified.

\subsection{$\mathcal{C}$-Bound}
\label{sec:c-bound}

\cite{Lacasse:2007} proposed to upper bound the Bayes risk by taking into account the mean and the variance of the prediction margin. \cite[Theorem 3]{Laviolette:2014} extended this bound to 
the multi-class case. We derive a new  bound in the probabilistic setting by explicitly taking into account the posterior probability of classes $P(Y~=~y~|~\mbf{X}=~\mbf{x})$ in the risk of the Bayes classifier~:
\begin{align*}
R(B_Q) &:= \E_{P(\mbf{X})}(r_{B_Q}(\mbf{x})) \\
&:=\E_{P(\mbf{X})}\left(\sum_{\substack{{i\in\{1,\dots,K\}}\\{i\neq B_Q(\mbf{x})}}}P(Y=i|\mbf{X}=\mbf{x})\right)
\end{align*}
where $r_{B_Q}(\mbf{x})$ denotes the Bayes risk in classifying an observation $\mbf{x}$.



\begin{thm}
\label{thm:prob-cbound}
Let $M$ be a random variable such that $[M|\mbf{X}=\mbf{x}]$ is a discrete random variable that is equal to the margin $M_Q(\mbf{x}, i)$ with probability $P(Y=i|\mbf{X}=\mbf{x})$, $i=\{1,\dots,K\}$. Let $\mu^{M}_1$ and $\mu^{M}_2$ be respectively the first and the second statistical moments of the random variable $M$.
Then, for all $Q$ on a hypothesis space $\mathcal{H}$, and for all distributions $P(\mbf{X})$ over $\mathcal{X}$ and $P(Y|\mbf{X})$ over $\mathcal{Y}$, such that $\mu^M_1>0$, we have:
\begin{align}
    R(B_Q) \leq 1 - \frac{(\mu^M_1)^2}{\mu^M_2}. \label{eq:prob-cbound}
\end{align}
\end{thm}
\begin{proof}
    At first, we show that $R(B_Q) = P(M\leq 0)$.
    One can notice that $r_{B_Q}(\mbf{x}) = P(M\leq 0|\mbf{X}=\mbf{x})$:
    \begin{align*}
        r_{B_Q}(\mbf{x})& = \sum_{\substack{{i\in\{1,\dots,K\}}\\{i\neq B_Q(\mbf{x})}}}P(Y=i|\mbf{X}=\mbf{x}) \\ 
        &= \sum_{i=1}^K P(Y=i|\mbf{X}=\mbf{x})\I{M_Q(\mbf{x},i)\leq 0} \\
        &= P(M\leq 0|\mbf{X}=\mbf{x}).  
    \end{align*}

    Applying the total probability law, we obtain:
    \begin{align}
    P(M\leq 0) &= \int_{\mathcal{X}} P(M\leq 0|\mbf{X}=\mbf{x}) P(\mbf{X}=\mbf{x})\diff\mbf{x} \nonumber\\ 
    &= \E_{P(\mbf{X})} P(M\leq 0|\mbf{X}=\mbf{x}) = R(B_Q). \label{eq:bayes-risk-via-prob-margins}
    \end{align}
    We remind the Cantelli-Chebyshev inequality:
    \begin{lem}
    \label{lem:cantelli-chebyshev}
        Let $Z$ be a random variable with the mean $\mu$ and the variance $\sigma^2$. Then, for every $a>0$, we have:
        \[
        P(Z\leq \mu - a) \leq \frac{\sigma^2}{\sigma^2 + a^2}.
        \]
    \end{lem}
    By taking $Z := M$, and $a = \mu^M_1$, we apply Lemma \ref{lem:cantelli-chebyshev} and deduce:
    \begin{align}
       P(M\leq 0) &\leq \frac{\mu^M_2 - (\mu^M_1)^2}{\mu^M_2 - (\mu^M_1)^2 + (\mu^M_1)^2 } = 1 - \frac{(\mu^M_1)^2}{\mu^M_2}. \label{eq:prob-M-less-0-bound}
    \end{align}
    Combining Eq. \eqref{eq:bayes-risk-via-prob-margins} and Eq. \eqref{eq:prob-M-less-0-bound} lead to the bound.
\end{proof}

Now, we show how to evaluate the Bayes risk in the case where we have an imperfect output $\hat{Y}$ with a different distribution than the true output $Y$. Further, we assume that the Bayes classifier is optimal in terms of risk minimization, i.e. it is equivalent to the maximum a posteriori rule:
\[
B_{Q_{\text{opt}}}(\mbf{x}) = \argmax_{c\in\mathcal{Y}}P(Y=c|\mbf{X}=\mbf{x}).
\]
Then, the Bayes risk for $Y$ and $\hat{Y}$ can be written as:
\begin{align*}
    R(B_{Q_{\text{opt}}}) &= \E_{P(\mbf{X})} r(X) = 1 - \max_{j=1,\dots, K}P(Y=j|\mbf{X});\\
    \hat{R}(B_{Q_{\text{opt}}}) &:= \E_{P(\mbf{X})} \hat{r}(X) := 1 - \max_{i=1,\dots, K}P(\hat Y=i|\mbf{X}).    
\end{align*}
Let the label imperfection be described by the following probability distribution:
\begin{align*}
    P(\hat Y=i|Y=j) &:= p_{i,j} \quad\forall(i,j)\in \{1,\dots,K\}^2\\
    \text{s.t.}\quad \sum_{i=1}^K p_{i,j} &= 1.
\end{align*}
We assume additionally that $P(\mbf{X}|Y) = P(\mbf{X}|Y, \hat Y)$.

\begin{prop}
\label{prop:tight-ci-bound}
Let $\hat{M}$ be a random variable such that $[\hat{M}|\mbf{X}=\mbf{x}]$ is a discrete random variable that is equal to the margin $M_Q(\mbf{x}, i)$ with probability $P(\hat{Y}=i|\mbf{X}=\mbf{x})$, $i=\{1,\dots,K\}$.
Let $\mu^{\hat{M}}_1$ and $\mu^{\hat{M}}_2$ be respectively the first and the second statistical moments of the random variable $\hat{M}$.
Then, for all distributions $P(\mbf{X})$ over $\mathcal{X}$ and $P(Y|\mbf{X})$, $P(\hat{Y}|\mbf{X})$ over $\mathcal{Y}$, such that $\mu^{\hat{M}}_1>0$, we have:
\begin{align}
     R(B_{Q_{\text{opt}}}) &\leq 1 - \frac{1}{\beta}\left[ \frac{(\mu^{\hat{M}}_1)^2}{\mu^{\hat{M}}_2}\right],\label{eq:prob-cbound-imperfect}
\end{align}
where $\beta := \max_{i=\{1,\dots,K\}} \sum_{j=1}^K p_{i,j}$.
\end{prop}

\begin{proof}
    To prove this proposition, we remind the inequality proposed by \cite[Section 3.2.1.]{Chittineni:1980}:
    \begin{align}
    \hat{R}(B_{Q_{\text{opt}}}) &\geq (1-\beta) + \beta R(B_{Q_{\text{opt}}}). \label{eq:imperfect-lower-bound}
    \end{align}
    Let $c_{\mbf{X}} := \max_{j=\{1,\dots,K\}}P(Y=j|\mbf{X})$ be the output of the Bayes classifier learnt on training examples with the true labels. Then, we have
    \begin{align}
    \hat{r}(\mbf{X}) &= 1 - \max_{i=\{1,\dots,K\}}\left[\sum_{j=1}^K p_{i,j}P(Y=j|\mbf{X})\right] \nonumber\\
    &\geq 1 - c_{\mbf{X}}\left[\max_{i=\{1,\dots,K\}}\left(\sum_{j=1}^K p_{i,j}\right)\right] \nonumber\\
    &= 1 - c_{\mbf{X}}\beta \nonumber\\
    &= (1 - \beta) + \beta r(\mbf{X}). \nonumber
    \end{align}
    By taking the expectation with respect to $P(\mbf{X})$, we obtain Inequality \eqref{eq:imperfect-lower-bound}.
    
    Thus, Inequality \eqref{eq:prob-cbound-imperfect} is directly obtained from Theorem \ref{thm:prob-cbound} and Inequality \eqref{eq:imperfect-lower-bound}.
\end{proof}

Proposition \ref{prop:tight-ci-bound} can be applied for semi-supervised learning: it upper bounds the risk when the training labeled set is augmented by imperfect pseudo-labeled  unlabeled training data. In comparison to the transductive bound of \cite{Feofanov:2019}, this bound is  tighter  in practice when the number of classes $K$ is greater than two, as we estimate the Bayes risk directly and not from the conditional risk. Given an imperfect label, our bound evaluates the $\mathcal{C}$-bound in this "noisy" case. Then, using the notion of $\beta$
, a correction of the bound is performed making the estimate more conservative. One can notice that $\beta=1$ 
 when there is no mislabeling.
 
\subsection{Bound for Feature Selection}

We propose another application of the bound to select relevant features given partially labeled data and a pseudo-labeling mechanism. The bound is then used as a selection criterion: we choose a feature subset that yields to the minimal value of the $\mathcal{C}$-bound estimated by training labeled and pseudo-labeled examples. 
However, the value of $\beta$ does not reflect the whole mislabeling matrix, but rather one of its row. This will have a negative impact on the selection criterion when the number of classes will be more than two.
To overcome this problem, we propose  the following theorem.

\begin{thm}
\label{thm:prob-cbound-imperfect}
Let $\hat{M}$ be a random variable such that $[\hat{M}|\mbf{X}=\mbf{x}]$ is a discrete random variable that is equal to the margin $M_Q(\mbf{x}, i)$ with probability $P(\hat{Y}=i|\mbf{X}=\mbf{x})$, $i=\{1,\dots,K\}$.
Let $\mu^{\hat{M}}_1$ and $\mu^{\hat{M}}_2$ be respectively the first and the second statistical moments of the random variable $\hat{M}$.
Then, for all distributions $P(\mbf{X})$ over $\mathcal{X}$ and $P(Y|\mbf{X})$, $P(\hat{Y}|\mbf{X})$ over $\mathcal{Y}$, such that $\mu^{\hat{M}}_1>0$, we have:
\begin{align}
    R(B_{Q_{\text{opt}}}) &\leq 1 - \frac{1}{\gamma}\left[ \frac{(\mu^{\hat{M}}_1)^2}{\mu^{\hat{M}}_2}\right],\label{eq:prob-cbound-imperfect-2}
\end{align}
where $\gamma := \sum_{j=1}^K \max_{i=\{1,\dots,K\}}p_{i,j}$.
\end{thm}

\begin{proof}
This Theorem is straightforward from Proposition \ref{prop:tight-ci-bound} by noticing that $\max_{i=\{1,\ldots,K\}} \sum_{j=1}^K p_{i,j} \leq \sum_{j=1}^K \max_{i=\{1,\ldots,K\}} p_{i,j}$.
\end{proof}

Despite Proposition \ref{prop:tight-ci-bound} provides a tighter bound, we have found that the use of Inequality \eqref{eq:prob-cbound-imperfect-2} is more stable as feature selection criterion.


\section{Wrapper with Imperfect Labels}
\label{sec:algorithm}
We present a new framework for wrapper feature selection using both labeled and unlabeled data  based on the probabilistic framework with mislabeling errors presented above.

\subsection{Framework}

The proposed \SW{} approach starts from a supervised Bayes classifier initially trained on available labeled examples. Then, it iteratively retrains the classifier by assigning pseudo-labels at each iteration to unlabeled examples that have prediction vote above a certain threshold.
 \cite{Feofanov:2019} proposed to obtain the threshold  $\bm{\theta} = (\theta_i)_{i=1}^K$ dynamically by minimizing the \emph{conditional Bayes error} $R_\mathcal{U|\boldsymbol{\theta}}(B_Q)$. We denote this algorithm further by \texttt{SLA}.
 



As a result, we obtain a new augmented training set that increases the diversity of training examples. However, the votes of the classifier are possibly biased and the pseudo-labeled examples contain mislabeling errors. 
In this sense, we propose a wrapper strategy that performs feature selection by modeling these mislabeling errors. In other words, we search a feature subset in the space of possible subsets that minimizes the $\mathcal{C}$-bound with imperfect labels (Theorem \ref{thm:prob-cbound-imperfect}). The bound is estimated using the augmented training set. In order to solve the optimization problem, we perform a heuristic search using a genetic algorithm. The whole strategy is depicted in Figure~\ref{fig:WFS}. 
\begin{figure}[t!]
\centering
\includegraphics[width=.9\textwidth]{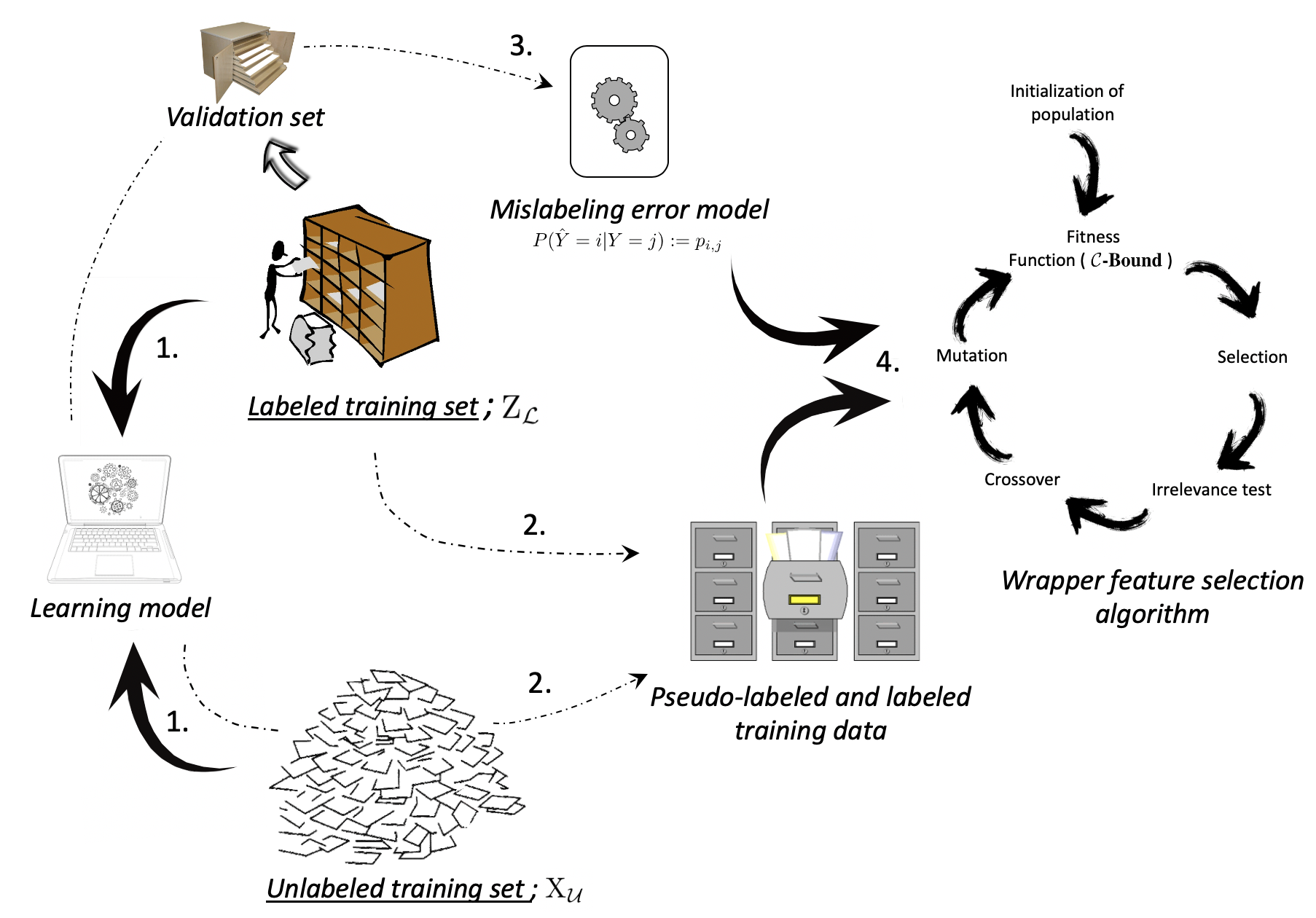}
    \caption{Pictorial pseudo-code of \SW. A Bayes classifier is first learned using labeled and unlabeled training sets. Unlabeled training data with predictions higher than an automatically found threshold are pseudo-labeled; and using a validation set, class probability errors of this pseudo-labeling scheme are estimated. Based on the augmented training data and the mislabeling probability errors a $\mathcal{C}$-Bound on the error of the Bayes classifier is found. This bound serves as a selection measure of the Fitness function used in a genetic algorithm for feature selection.  }
 \label{fig:WFS}
\end{figure}

\subsection{Genetic Algorithm and Its Limitations}

A genetic algorithm \cite{Goldberg:1988} is an evolutionary optimization algorithm inspired by the natural selection process. A \textit{fitness function} is optimized by evolving iteratively a \textit{population} of \textit{candidates} (in our case, binary representation of possible feature subsets).

Starting from a randomly drawn population, the algorithm produces iteratively new populations, called \textit{generations}, by preserving  \textit{parents}, $p$ candidates with best fitness, and creating \textit{offspring} from parents using operation of crossover and mutation (Figure \ref{fig:ga-new-child}). After a predefined number of generations the algorithm is stopped, and a candidate with the best fitness in the last population is returned. Further, we call this algorithm as the classic genetic algorithm,  \texttt{CGA}.

\texttt{CGA} can be very effective for the wrapper feature selection when the number of features is very large. However, there might be several limitations. Firstly, the algorithm may have a large variance in giving results depending on the initialization of population. Therefore, it needs usually a large number of generations to have a stable output.

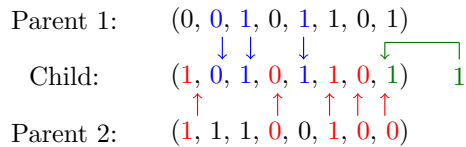
\begin{figure}[ht!]
    \centering
    \begin{tikzpicture}[scale=1.5, line width=0.15mm]
    
    \node at (0, 0) {(0, \blue{0}, \blue{1}, 0, \blue{1}, 1, 0, 1)};
    \node at (0, -0.5) {(\red{1}, \blue{0}, \blue{1}, \red{0}, \blue{1}, \red{1}, \red{0}, \green{1})};
    \node at (0, -1) {(\red{1}, 1, 1, \red{0}, 0, \red{1}, \red{0}, \red{0})};
    
    \node at (-2, 0) {Parent 1:};
    \node at (-2, -0.5) {Child:};
    \node at (-2, -1) {Parent 2:};
    
    \draw[->, color=blue] (0.12, -0.15) -- (0.12, -0.35);
    \draw[->, color=blue] (-0.6, -0.15) -- (-0.6, -0.35);
    \draw[->, color=blue] (-0.35, -0.15) -- (-0.35, -0.35);
    
    \draw[->, color=red] (-0.82, -0.85) -- (-0.82, -0.65);
    \draw[->, color=red] (-0.11, -0.85) -- (-0.11, -0.65);
    \draw[->, color=red] (0.35, -0.85) -- (0.35, -0.65);
    \draw[->, color=red] (0.6, -0.85) -- (0.6, -0.65);
    \draw[->, color=red] (0.84, -0.85) -- (0.84, -0.65);
    
    \draw[->, color=mygreen] (1.5, -0.2) -| (0.84, -0.35);
    \draw[-, color=mygreen] (1.5, -0.2) -- (1.5, -0.35);
    
    \node at (1.5, -0.5) {\green{1}};
    
    \end{tikzpicture}
    \caption{A simple scheme of how a new child is generated from two parents. The crossover procedure (red and blue colors) is followed by mutation (green color).}
    \label{fig:ga-new-child}
\end{figure}

Another problem would be that during the crossover a child inherits features from the parents at random; ignoring any information like feature importance. Because of that, a solution the algorithm outputs is generally not sparse as it could be. To produce a sparse solution, it is usually spread to limit the search space by fixing the number of selected features \cite{Persello:2016}. However, it is not very clear which number of features should be taken.

\subsection{Feature Selection Genetic Algorithm}

In this section, we describe a new genetic algorithm  for  feature selection in semi-supervised learning. The main idea of the algorithm is to take into account the importance of features during the generation of  a new population. It allows to output a sparse solution preserving discriminative power and not fixing the number of features.

First, we initialize the population $\mathcal{P}_0$ by randomly generating feature subsets of a fixed length. In our experiments, this length is equal to $\floor{\sqrt{d}}$. Each candidate $S\in\mathcal{P}_0$ is a feature subset (\textbf{initialization}). Then, for the generation $g\geq 0$ and for each candidate $S \in \mathcal{P}_g$, we train a supervised model and compute a score reflecting the strength of the subset. We  derive weights $w^{S}_1, \dots, w^{S}_d$ for each feature, using ensemble methods based on decision trees (\textbf{fitness and features weights computation}).

\noindent To accelerate the convergence and reduce the variance of the algorithm, we embed a test to eliminate irrelevant to response variables (\textbf{feature relevant test}). This idea bears similarity with the work of \cite{Tuv:2009} where variables are compared with their copies using randomly permuted values. For each feature, we compute the average weights:
    \[
    \bar{w}_t = \frac{\sum_{S: [S\in \mathcal{P}_g]\wedge[t\in S]}w_t^{S}}{\sum_{\tau=1}^d\sum_{S: [S\in \mathcal{P}_g]\wedge[\tau \in S]}w_\tau ^{S}},\quad t\in\{1,\dots,d\}.
    \]
    We detect suspicious irrelevant features that have average weights less than a fixed threshold $\theta_{\text{out}}$: $S_{\text{out}} = \{t: \bar{w}_t\leq \theta_{\text{out}}, t\in\{1,\dots,d\}\}$. 
    A new supervised model is learnt on a new data set, composed by: the relevant features detected so far by the best parent, suspicious irrelevant features and a randomly permuted copy of those suspicious irrelevant variables. If the difference between the weight of a suspicious  feature that belongs to $S_{\text{out}}$ and the weight of its noisy counterpart is not significant, the suspicious irrelevant feature is called irrelevant, removed and will not be further considered by the algorithm.

    \noindent Among the population $\mathcal{P}_g$, $p$ candidates with best fitness are selected, preserved for the next population $\mathcal{P}_{g+1}$ and used to produce new offspring (\textbf{parent selection}).

    \noindent A new child is generated by mating two parents (\textbf{crossover}). In contrast to \texttt{CGA}, we inherit variables according to their weights: for each parent, its features are sorted by their weights in the decreasing order. The crossover point that characterizes the proportion of features inherited from the first parent is taken randomly, and we fill the child by its sorted features until we reach the quota. The rest of the features are taken from the second parent under a condition that there are no repetitions.

    \noindent To increase the diversity of candidates, we perform mutation of children in the same way as in \texttt{CGA}. In addition, we define a possibility to mutate the number of features in the subset. For each child, its length can be randomly increased, decreased or remain the same (\textbf{mutation}).
    
    \noindent We repeat these steps for several generations until we obtain the final population. Since we enforce our algorithm to start with $\sqrt{d}$ features, this number might be too small in some cases. To overcome this, we output a final candidate by  combining outputs of all candidates from the final population. The features are taken by voting and we threshold the ones that were present in few candidate subsets.
    
\begin{table*}[ht!]
\caption{Characteristics of data sets used in our experiments ordered by dimension $d$.}
\label{tab:data-set-description}
\centering
\hfill \break
\scalebox{0.8}
{
      \begin{tabular}{cccccc}
        \toprule
        Data set & \# of lab. examples, & \# of unlab.  examples, & \# of test examples, & Dimension, & \# of classes, \\
                 & $l$ & $u$ & $n_t$ & $d$ & $K$ \\
        \midrule
        \texttt{Protein} & 97 & 875 & 108 & 77 & 8 \\
        \texttt{Madelon} & 234 & 2106 & 260 & 500 & 2 \\ 
        \texttt{Isolet} & 140 & 1264 & 156 & 617 & 26 \\                
        \texttt{Fashion} & 99 & 9801 & 100 & 784 & 10 \\
        \texttt{MNIST} & 99 & 9801 & 100 & 784 & 10 \\
        \texttt{Coil20} & 130 & 1166 & 144 & 1024 & 20 \\
        \texttt{PCMAC} & 175 & 1574 & 194 & 3289 & 2 \\
        \texttt{RELATHE} & 128 & 1156 & 143 & 4322 & 2 \\
        \texttt{BASEHOCK} & 180 & 1614 & 199 & 4862 & 2 \\
        \texttt{Gisette} & 69 & 6861 & 70 & 5000 & 2 \\
        \bottomrule
      \end{tabular}}
\end{table*}

\section{Experimental Results}
\label{sec:exp}
To illustrate the benefit of our approach, we conducted a number of experiments  on 10 publicly available data sets \cite{Chang:2011,Guyon:2003:Design,Li:2018}. The associated applications are image recognition, with \texttt{Fashion},  \texttt{MNIST},  \texttt{Coil20} and \texttt{Gisette} data sets;  text classification databases  \texttt{PCMAC}, \texttt{RELATHE}, \texttt{BASEHOCK};  bioinformatics with \texttt{Protein} data set; feature selection with \texttt{Madelon}; and   a speech recognition task with \texttt{Isolet}.

The main characteristics of all data sets are summarized in Table~\ref{tab:data-set-description}. Since we are interested in practical use of the algorithm, we test the algorithms under the condition that $l\ll u$. For \texttt{MNIST} and \texttt{Fashion} data sets, we consider its subset of 10000 observations.

In all the experiments, we consider the Random Forest algorithm \cite{Breiman:2001}, denoted as \texttt{RF}, with 200 trees and the maximal depth of trees, as the Bayes classifier with the uniform posterior distribution. For an observation $\mbf{x}$, we evaluate the vector of class votes $\{v(\mbf{x}, i)\}_{i=1}^K$ by averaging over the trees the vote given to each class by the tree. A tree computes a class vote as the fraction of training examples in a leaf belonging to a class.

In all experiments, we use the \texttt{SLA} to evaluate the quality of selection. In other words, at first, we find a feature subset using a feature selection method, then we train \texttt{SLA} on the selected features and compute its performance.

For all experimental results, we perform 20 random (train/unlabeled/test) splits of the initial collection and report the average classification accuracy over the 20 trials on the unlabeled training set (\texttt{ACC-U}). In addition, for the final experiments, we report also the average accuracy on the test set (\texttt{ACC-T}). We set a time limit to 1 hour per split and terminate an algorithm if the limit is exceeded. These cases are indicated as \texttt{NA}.

We present results using the following notations: the bold face is used to emphasize the highest performance rate; by the symbol $\downarrow$ we indicate that the performance is significantly worse compared to the best result, according to Mann-Whitney U test \cite{Mann:1947} at the p-value level equal to 0.01.

\subsection{Impact of the Selection Criterion}

In Section \ref{sec:framework}, we propose to use the $\mathcal{C}$-bound with the imperfect labels (Theorem \ref{thm:prob-cbound-imperfect}), further denoted by \texttt{CBIL}, as a criterion to be minimized inside the wrapper. To validate this idea, we compare it with the following criteria: the Out-Of-Bag  error \texttt{OOB} of the Random Forest  \cite{Breiman:2001}, where for each tree the validation set comprises the training examples that were not included in the bootstrap sample; the transductive bound \texttt{TB} of \cite{Feofanov:2019};  and
the $\mathcal{C}$-bound computed without the $\gamma$-correction \texttt{CB}.

To evaluate the $\mathcal{C}$-bound, we approximate probabilities $P(Y=i|\mbf{X}=\mbf{x}),\ i\in\{1,\dots,K\}$ by the class vote $v(\mbf{x}, i)$ reflecting the confidence in predicting $i$. We estimate $\gamma$ on the labeled set by comparing the true labels with the corresponding out-of-bag predictions.

To eliminate the effect of the search scheme, for each data set, we have randomly chosen 40 feature subsets, and exhaustively performed feature selection using each of the aforementioned criterion. This setup corresponds to the initial step of a genetic algorithm. Then, we train \texttt{SLA} on the selected features and compute its performance.

\setlength{\tabcolsep}{0.65em}
\setlength\extrarowheight{3pt}
\begin{table}[ht!]
\caption{The classification performance on the unlabeled set of different selection criteria at the initialization step of \texttt{FSGA}. The ground truth \texttt{GT} represents the accuracy of the best possible choice among considered feature subsets.}
\label{tab:real-data-metrics-exp-res}

\hfill \break
\centering
\scalebox{0.58}{
\begin{tabular}{c||c|| c| c| c| c}
\toprule

Data set  & \texttt{GT} & \texttt{OOB} & 
\texttt{TB} & \texttt{CB} & \texttt{CBIL}\\

\midrule
\texttt{Protein}  & 
\textit{.665 $\pm$ .025 }& .615 $\pm$ .036 & .441$^\downarrow$ $\pm$ .06 & .613 $\pm$ .03 & \textbf{.622} $\pm$ .03\\
\midrule
\texttt{Madelon}  & 
\textit{.589 $\pm$ .016} & .564 $\pm$ .034 & .503$^\downarrow$ $\pm$ .014 & .568 $\pm$ .029 & \textbf{.57} $\pm$ .028\\
\midrule
\texttt{Isolet}  & 
\textit{.678 $\pm$ .019} & \textbf{.644} $\pm$ .039 & .435$^\downarrow$ $\pm$ .02 & .625 $\pm$ .027 & .633 $\pm$ .035\\
\midrule
\texttt{Fashion}  & 
\textit{.619 $\pm$ .014 }& .57 $\pm$ .031 & .543$^\downarrow$ $\pm$ .022 & .575 $\pm$ .029 & \textbf{.58} $\pm$ .023\\
\midrule
\texttt{MNIST}  & 
\textit{.549 $\pm$ .012} & .377$^\downarrow$ $\pm$ .028 & \textbf{.499} $\pm$ .034 & .451$^\downarrow$ $\pm$ .036 & .463$^\downarrow$ $\pm$ .038\\
\midrule
\texttt{Coil20}  &
\textit{.918 $\pm$ .012} & \textbf{.908} $\pm$ .021 & .866$^\downarrow$ $\pm$ .023 & .891 $\pm$ .025 & .898 $\pm$ .023\\
\midrule
\texttt{PCMAC}  & 
\textit{.62 $\pm$ .011} & .574$^\downarrow$ $\pm$ .032 & .54$^\downarrow$ $\pm$ .021 & \textbf{.61} $\pm$ .024 & \textbf{.61} $\pm$ .02\\
\midrule
\texttt{RELATHE}  & 
\textit{.643 $\pm$ .012} & .592 $\pm$ .021 & .577$^\downarrow$ $\pm$ .015 & \textbf{.598} $\pm$ .021 & \textbf{.598
} $\pm$ .025\\
\midrule
\texttt{BASEHOCK}  & 
\textit{.69 $\pm$ .011} & .655 $\pm$ .041 & .586$^\downarrow$ $\pm$ .01 & \textbf{.666} $\pm$ .036 & .658 $\pm$ .034\\
\midrule
\texttt{Gisette}  & 
\textit{.83 $\pm$ .011} & \textbf{.814} $\pm$ .011 & .622$^\downarrow$ $\pm$ .033 & .81 $\pm$ .014 & .813 $\pm$ .013\\
\bottomrule

\end{tabular}}
\end{table}

The classification accuracy on the unlabeled set is given in Table \ref{tab:real-data-metrics-exp-res}. At first, we observe that selection by the transductive bound leads to the significantly worst performance on all data sets except \texttt{MNIST}. This fact indicates that the minimization of this bound does not fit the feature selection procedure. In reality, this bound grows together with the margin mean ignoring the possible increase of the margin's variance.

In contrast, the $\mathcal{C}$-bound (Inequality \eqref{eq:prob-cbound}) finds the solution that both maximizes the margin mean and minimizes variance of the margin. Overall, we can see that the $\mathcal{C}$-bound and \texttt{OOB} (which is an unbiased error estimate of \texttt{RF} in the supervised case) provide comparable results in average on all data sets. However, due to the presence of imperfect labels, the criteria are not stable in their selection. The sharp decrease in performance is observed for \texttt{OOB} on \texttt{MNIST} and \texttt{PCMAC} data sets, whereas for $\mathcal{C}$-bound on \texttt{Coil20} and \texttt{Isolet} data sets.

Indeed, \texttt{OOB} error may reflect the strength of a feature subset wrongly, since it is misled by the validation set, which may contain unlabeled examples with wrongly predicted labels. In turn, the estimate of $\mathcal{C}$-bound can be strongly biased when the estimation of the margins is too optimistic for the pseudo-labeled examples.

In this connection, \texttt{CBIL} provides a safer solution by penalizing the $\mathcal{C}$-bound using $\gamma$ estimated using the labeled examples only. This criterion becomes a comprise of semi-supervised pseudo-labeled data and  supervised regularizer. From the results, we observe that \texttt{CBIL} provides safe results among all data sets and improves \texttt{CB} in 7/10 cases. For the following experiments, we use \texttt{CBIL} as a criterion of our proposed wrapper.

\subsection{Impact of the Wrapper's Search Scheme}
In this section, we study the performance of the wrapper \texttt{SEW-IL} under the use of different search algorithms. We validate our approach to select features by genetic algorithm \texttt{FSGA} by comparing with the classical genetic algorithm \texttt{CGA} and the forward sequential search algorithm  \texttt{FSS}. 
Due to large run-time of \texttt{FSS}, at each step we add 10\% best features into the model. 
For the genetic algorithms, the number of generations is set to 20, the population size is 40 and the number of parents is set to 8.

The performance results are described in Table \ref{tab:real-data-optim-exp-res}. Sequential search algorithm performs significantly worse and becomes computationally abundant on  large-scale data sets. This approves our strategy to select features by a genetic algorithm.

One can observe that \texttt{CGA} performs significantly better on 4 data sets. \texttt{FSGA} performs significantly better on \texttt{Madelon} data set, which is particularly interesting due to a very small number of informative features (20 out of 500 \cite{Guyon:2003:Design}). It validates the feature relevant test of \texttt{FSGA} that gives improvement compared to \texttt{CGA}.

Despite the high performance of \texttt{CGA},  approximately half of the features are kept. In many applications, it is desired to obtain a solution that would be sparse and/or interpretative. 
As it can be seen, \texttt{FSGA} outputs sparser feature subsets and its performance is not inferior to \texttt{CGA} in most of cases. From this we can conclude that \texttt{FSGA} provides better approximation of the minimal Markov blanket.

\setlength{\tabcolsep}{0.65em}
\setlength\extrarowheight{3pt}
\begin{table}[ht!]
\caption{The classification performance of the wrapper \texttt{SEW-IL} with different search schemes. The number of features (averaged over 20 trials), \texttt{ACC-U} and \texttt{ACC-T} for each method are indicated. }
\label{tab:real-data-optim-exp-res}

\hfill \break
\centering
\scalebox{0.58}{
\begin{tabular}{c|c||c|c| c|c| c|c}
\toprule

\multirow{3}{*}{Data set} & \multirow{3}{*}{Score} & 
\multicolumn{6}{c}{\texttt{SEW-IL}}\\
\cline{3-8}
& & \multicolumn{2}{c|}{\texttt{FSS}} & \multicolumn{2}{c|}{\texttt{CGA}} & \multicolumn{2}{c}{\texttt{FSGA}}\\
\cline{3-8}
& & \texttt{ACC} & $d'$
& \texttt{ACC} & $d'$
& \texttt{ACC} & $d'$\\

\midrule
\multirow{2}{*}{\texttt{Protein}} & $\mathtt{ACC-U}$ & 
.733$^\downarrow$ $\pm$ .028 & \multirow{2}{*}{19} & \textbf{.761} $\pm$ .03 & \multirow{2}{*}{41} & .757 $\pm$ .026 & \multirow{2}{*}{33}\\

 & $\mathtt{ACC-T}$ & 
.73 $\pm$ .042 &  & .751 $\pm$ .043 &  & \textbf{.753} $\pm$ .043 & \\
\midrule
\multirow{2}{*}{\texttt{Madelon}} & $\mathtt{ACC-U}$ & 
.585$^\downarrow$ $\pm$ .043 & \multirow{2}{*}{63} & .584$^\downarrow$ $\pm$ .021 & \multirow{2}{*}{250} & \textbf{.666} $\pm$ .042 & \multirow{2}{*}{27}\\

 & $\mathtt{ACC-T}$ & 
.591$^\downarrow$ $\pm$ .053 &  & .604$^\downarrow$ $\pm$ .042 &  & \textbf{.662} $\pm$ .046 & \\
\midrule
\multirow{2}{*}{\texttt{Isolet}} & $\mathtt{ACC-U}$ & 
.78$^\downarrow$ $\pm$ .027 & \multirow{2}{*}{73} & \textbf{.837} $\pm$ .017 & \multirow{2}{*}{313} & .814$^\downarrow$ $\pm$ .017 & \multirow{2}{*}{66}\\

 & $\mathtt{ACC-T}$ & 
.778$^\downarrow$ $\pm$ .038 &  & \textbf{.838} $\pm$ .03 &  & .802$^\downarrow$ $\pm$ .033 & \\
\midrule
\multirow{2}{*}{\texttt{Fashion}} & $\mathtt{ACC-U}$ & 
.625$^\downarrow$ $\pm$ .029 & \multirow{2}{*}{86} & \textbf{.691} $\pm$ .017 & \multirow{2}{*}{392} & .668$^\downarrow$ $\pm$ .018 & \multirow{2}{*}{100}\\
 & $\mathtt{ACC-T}$ &
.617$^\downarrow$ $\pm$ .033 &  & \textbf{.686} $\pm$ .033 &  & .657$^\downarrow$ $\pm$ .026 & \\
\midrule
\multirow{2}{*}{\texttt{MNIST}} & $\mathtt{ACC-U}$ & 
.724$^\downarrow$ $\pm$ .021 & \multirow{2}{*}{86} & \textbf{.818} $\pm$ .022 & \multirow{2}{*}{396} & .786$^\downarrow$ $\pm$ .018 & \multirow{2}{*}{91}\\
 & $\mathtt{ACC-T}$ &
.724$^\downarrow$ $\pm$ .047 &  & \textbf{.829} $\pm$ .049 &  & .794$^\downarrow$ $\pm$ .046 & \\

\midrule
\multirow{2}{*}{\texttt{Coil20}} & $\mathtt{ACC-U}$ &
.896$^\downarrow$ $\pm$ .018 & \multirow{2}{*}{102} & .938 $\pm$ .012 & \multirow{2}{*}{506} & \textbf{.941} $\pm$ .011 & \multirow{2}{*}{102}\\
  & $\mathtt{ACC-T}$ &

.889$^\downarrow$ $\pm$ .025 &  & .933 $\pm$ .023 &  & .\textbf{939} $\pm$ .024 & \\
\midrule

\multirow{2}{*}{{\texttt{PCMAC}}} & $\mathtt{ACC-U}$ & 
.638$^\downarrow$ $\pm$ .066 & \multirow{2}{*}{222} & \textbf{.818} $\pm$ .019 & \multirow{2}{*}{1647} & .813 $\pm$ .022 & \multirow{2}{*}{57}\\
 & $\mathtt{ACC-T}$ &
.615$^\downarrow$ $\pm$ .098 &  & \textbf{.847} $\pm$ .031 &  & .83 $\pm$ .035 & \\

\midrule
\multirow{2}{*}{{\texttt{RELATHE}}} & $\mathtt{ACC-U}$ & 
.728$^\downarrow$ $\pm$ .024 & \multirow{2}{*}{266} & \textbf{.773} $\pm$ .026 & \multirow{2}{*}{2169} & .74$^\downarrow$ $\pm$ .036 & \multirow{2}{*}{71}\\
 & $\mathtt{ACC-T}$ &
.708$^\downarrow$ $\pm$ .047 &  & \textbf{.78} $\pm$ .049 &  & .726$^\downarrow$ $\pm$ .054 & \\

\midrule
\multirow{2}{*}{{\texttt{BASEHOCK}}} & $\mathtt{ACC-U}$ & 
\texttt{NA} & \multirow{2}{*}{\texttt{NA}} & .91 $\pm$ .012 & \multirow{2}{*}{2433} & \textbf{.915} $\pm$ .007 & \multirow{2}{*}{75}\\

 & $\mathtt{ACC-T}$ &
\texttt{NA} &  & \textbf{.923} $\pm$ .024 &  & .912 $\pm$ .024 & \\

\midrule
\multirow{2}{*}{{\texttt{Gisette}}} & $\mathtt{ACC-U}$ & 
\texttt{NA} & \multirow{2}{*}{\texttt{NA}} & \textbf{.88} $\pm$ .014 & \multirow{2}{*}{2487} & .872 $\pm$ .015 & \multirow{2}{*}{69}\\
& $\mathtt{ACC-T}$ & 
\texttt{NA} &  & \textbf{.884} $\pm$ .041 &  & .872 $\pm$ .033 & \\

\bottomrule

\end{tabular}}
\end{table}

\subsection{Comparison with the State-of-the-Art}

\setlength{\tabcolsep}{0.65em}
\setlength\extrarowheight{3pt}
\begin{table*}[ht!]
\caption{The classification performance on the unlabeled and the test sets (\texttt{ACC-U} and \texttt{ACC-T} respectively) of various data sets presented in Table \ref{tab:data-set-description}. In addition, the number of features (averaged over 20 trials) used for learning by each method is indicated. $^\downarrow$ indicates statistically significantly worse performance than the best result (shown in bold), according to a Mann-Whitney U test ($p < 0.01)$.}
\label{tab:real-data-sota-exp-res}

\hfill \break
\centering
\scalebox{0.65}{
\begin{tabular}{c|c||c|c| c|c| c|c| c|c| c|c}
\toprule

\multirow{2}{*}{Data set} & \multirow{2}{*}{Score} & \multicolumn{2}{c|}{\texttt{RLSR}} & 
\multicolumn{2}{c|}{\texttt{SFS}} & 
\multicolumn{2}{c|}{\texttt{SSLS}} & 
\multicolumn{2}{c|}{\texttt{CoT-FSS}} & \multicolumn{2}{c}{\texttt{SEW-IL}}\\
\cline{3-12}
& & \texttt{ACC} & $d'$
& \texttt{ACC} & $d'$
& \texttt{ACC} & $d'$
& \texttt{ACC} & $d'$
& \texttt{ACC} & $d'$\\

\midrule
\multirow{2}{*}{\texttt{Protein}} & $\mathtt{ACC-U}$ & 
.727$^\downarrow$ $\pm$ .024 & \multirow{2}{*}{19} & .712$^\downarrow$ $\pm$ .028 & \multirow{2}{*}{19} & .685$^\downarrow$ $\pm$ .028 & \multirow{2}{*}{19} & .715$^\downarrow$ $\pm$ 
.039 & \multirow{2}{*}{19} & \textbf{.757} $\pm$ .026 & \multirow{2}{*}{33}\\
& $\mathtt{ACC-T}$ & 

.721 $\pm$ .046 &  & .716 $\pm$ .049 &  & .673$^\downarrow$ $\pm$ .048 &  & .715$^\downarrow$ $\pm$ .05 &  & \textbf{.753} $\pm$ .043 & \\

\midrule
\multirow{2}{*}{\texttt{Madelon}} & $\mathtt{ACC-U}$ & 
.558$^\downarrow$ $\pm$ .036 & \multirow{2}{*}{63} & .589$^\downarrow$ $\pm$ .028 & \multirow{2}{*}{63} & .553$^\downarrow$ $\pm$ .04 & \multirow{2}{*}{63} & .54$^\downarrow$ $\pm$ .042 & \multirow{2}{*}{63} & \textbf{.666} $\pm$ .042 & \multirow{2}{*}{27}\\

 & $\mathtt{ACC-T}$ & 
 .564$^\downarrow$ $\pm$ .062 &  & .601$^\downarrow$ $\pm$ .032 &  & .566$^\downarrow$ $\pm$ .06 &  & .531$^\downarrow$ $\pm$ .062 &  & \textbf{.662} $\pm$ .046 & \\
 
\midrule
\multirow{2}{*}{\texttt{Isolet}} & $\mathtt{ACC-U}$ & 
\textbf{.822} $\pm$ .02 & \multirow{2}{*}{73} & .672$^\downarrow$ $\pm$ .022 & \multirow{2}{*}{73} & .666$^\downarrow$ $\pm$ .016 & \multirow{2}{*}{73} & .658$^\downarrow$ $\pm$ .058 & \multirow{2}{*}{73} & .814 $\pm$ .017 & \multirow{2}{*}{66}\\

 & $\mathtt{ACC-T}$ & 
 \textbf{.814} $\pm$ .029 &  & .659$^\downarrow$ $\pm$ .037 &  & .649$^\downarrow$ $\pm$ .036 &  & .648$^\downarrow$ $\pm$ .071 &  & .802 $\pm$ .033 & \\
 
\midrule
\multirow{2}{*}{\texttt{Fashion}} & $\mathtt{ACC-U}$ & 
.591$^\downarrow$ $\pm$ .016 & \multirow{2}{*}{86} & .528$^\downarrow$ $\pm$ .034 & \multirow{2}{*}{86} & .512$^\downarrow$ $\pm$ .031 & \multirow{2}{*}{86} & \texttt{NA} & \multirow{2}{*}{\texttt{NA}} & \textbf{.668} $\pm$ .018 & \multirow{2}{*}{100}\\

 & $\mathtt{ACC-T}$ &
 .59$^\downarrow$ $\pm$ .041 &  & .52$^\downarrow$ $\pm$ .054 &  & .504$^\downarrow$ $\pm$ .045 &  & \texttt{NA} &  & \textbf{.657} $\pm$ .026 & \\
 
\midrule
\multirow{2}{*}{\texttt{MNIST}} & $\mathtt{ACC-U}$ & 
.21$^\downarrow$ $\pm$ .022 & \multirow{2}{*}{86} & .111$^\downarrow$ $\pm$ .002 & \multirow{2}{*}{86} & .445$^\downarrow$ $\pm$ .062 & \multirow{2}{*}{86} & \texttt{NA} & \multirow{2}{*}{\texttt{NA}} & \textbf{.786} $\pm$ .018 & \multirow{2}{*}{91}\\

 & $\mathtt{ACC-T}$ &
.212$^\downarrow$ $\pm$ .03 &  & .109$^\downarrow$ $\pm$ .003 &  & .451$^\downarrow$ $\pm$ .052 &  & \texttt{NA} &  & \textbf{.794} $\pm$ .046 & \\

\midrule
\multirow{2}{*}{\texttt{Coil20}} & $\mathtt{ACC-U}$ &
.922$^\downarrow$ $\pm$ .013 & \multirow{2}{*}{102} & .81$^\downarrow$ $\pm$ .015 & \multirow{2}{*}{102} & .813$^\downarrow$ $\pm$ .018 & \multirow{2}{*}{102} & .843$^\downarrow$ $\pm$ .069 & \multirow{2}{*}{102} & \textbf{.941} $\pm$ .011 & \multirow{2}{*}{102}\\

  & $\mathtt{ACC-T}$ &
.916$^\downarrow$ $\pm$ .025 &  & .816$^\downarrow$ $\pm$ .025 &  & .809$^\downarrow$ $\pm$ .023 &  & .832$^\downarrow$ $\pm$ .083 &  &\textbf{ .939} $\pm$ .024 & \\

\midrule

\multirow{2}{*}{{\texttt{PCMAC}}} & $\mathtt{ACC-U}$ & 
\textbf{.817} $\pm$ .021 & \multirow{2}{*}{222} & .726$^\downarrow$ $\pm$ .047 & \multirow{2}{*}{222} & .595$^\downarrow$ $\pm$ .057 & \multirow{2}{*}{222} & \texttt{NA} & \multirow{2}{*}{\texttt{NA}} & .813 $\pm$ .022 & \multirow{2}{*}{57}\\

& $\mathtt{ACC-T}$ &
.825 $\pm$ .037 &  & .727$^\downarrow$ $\pm$ .061 &  & .598$^\downarrow$ $\pm$ .066 &  & \texttt{NA} &  & \textbf{.83} $\pm$ .035 & \\

\midrule
\multirow{2}{*}{{\texttt{RELATHE}}} & $\mathtt{ACC-U}$ & 

\textbf{.759} $\pm$ .019 & \multirow{2}{*}{266} & .658$^\downarrow$ $\pm$ .033 & \multirow{2}{*}{266} & .606$^\downarrow$ $\pm$ .02 & \multirow{2}{*}{266} & \texttt{NA} & \multirow{2}{*}{\texttt{NA}} & .74 $\pm$ .036 & \multirow{2}{*}{71}\\

 & $\mathtt{ACC-T}$ &
\textbf{.757} $\pm$ .045 &  & .647$^\downarrow$ $\pm$ .037 &  & .607$^\downarrow$ $\pm$ .028 &  & \texttt{NA} &  & .726 $\pm$ .054 & \\

\midrule
\multirow{2}{*}{{\texttt{BASEHOCK}}} & $\mathtt{ACC-U}$ & 
.908 $\pm$ .015 & \multirow{2}{*}{287} & .83$^\downarrow$ $\pm$ .039 & \multirow{2}{*}{287} & .656$^\downarrow$ $\pm$ .075 & \multirow{2}{*}{287} & \texttt{NA} & \multirow{2}{*}{\texttt{NA}} & \textbf{.915} $\pm$ .007 & \multirow{2}{*}{75}\\

 & $\mathtt{ACC-T}$ &
\textbf{.912} $\pm$ .024 &  & .831$^\downarrow$ $\pm$ .053 &  & .666$^\downarrow$ $\pm$ .086 &  & \texttt{NA} &  & \textbf{.912} $\pm$ .024 & \\

\midrule
\multirow{2}{*}{{\texttt{Gisette}}} & $\mathtt{ACC-U}$ & 
.669$^\downarrow$ $\pm$ .084 & \multirow{2}{*}{293} & \textbf{.877} $\pm$ .012 & \multirow{2}{*}{293} & .615$^\downarrow$ $\pm$ .041 & \multirow{2}{*}{293} & \texttt{NA} & \multirow{2}{*}{\texttt{NA}} & .872 $\pm$ .015 & \multirow{2}{*}{69}\\

& $\mathtt{ACC-T}$ & 
.683$^\downarrow$ $\pm$ .086 &  & \textbf{.874} $\pm$ .035 &  & .614$^\downarrow$ $\pm$ .059 &  & \texttt{NA} &  & .872 $\pm$ .033 & \\
\bottomrule

\end{tabular}}
\end{table*}

In our experiments we compare our approach \texttt{SEW-IL} optimized by \texttt{FSGA} with the following methods: an embedded selection by rescaled linear square regression \texttt{RLSR} \cite{Chen:2017}; a feature ranking by Semi\_Fisher score \texttt{SFS} \cite{Yang:2010}; a feature ranking by semi-supervised Laplacian score \texttt{SSLS} \cite{Zhao:2008}; a semi-supervised wrapper using co-training and the forward sequential search algorithm \texttt{CoT-FSS} \cite{Ren:2008}.

The hyperparameters of all methods are set to their default values as there are not enough labeled training samples to tune them correctly. Specifically, $\gamma$ for \texttt{RLSR} is set to 0.1; the number of nearest neighbors is set to 20 for \texttt{SSLS} and \texttt{SFS}. Implementing the \texttt{CoT-FSS}, we take the \texttt{RF} as the base classifier and \texttt{OOB} as the selection criterion, since the paper does not go beyond the supervised selection criteria.


For \texttt{RLSR}, \texttt{SFS}, \texttt{SSLS} and \texttt{CoT-FSS}, we fix the number of selected features as $\ceil{d^{2/3}}$. Table \ref{tab:real-data-sota-exp-res} summarizes the performance results and  the number of features.

From the results we observe that \texttt{SEW-IL} compares well to other methods and significantly outperforms the others  wrt \texttt{ACC-U} and \texttt{ACC-T} on 4 data sets. On the large dimensional sets \texttt{PCMAC, RELATHE, BASEHOCK, Gisette} the method provides an extremely sparse output not being significantly worse. In addition, in cases when the initial number of features ($\sqrt{d}$) would not be sufficient, the algorithm may increase this number through propagating mutated candidates or by combining outputs at the final stage.

Compared to our approach, another wrapper \texttt{CoT-FSS} has large computational time and becomes infeasible on the large-scale data sets not performing well on the others. We can also see that the performance of \texttt{RLSR} significantly fluctuates from one data set to another. This could be connected with its sensitivity to the value of its regularization parameter $\gamma$, which can not be tuned in practice.

The filter methods, \texttt{SFS} and \texttt{SSLS}, in most of situations are significantly worse than \texttt{SEW-IL}. This may be connected with the fact that they score variables independently, so in the presence of many redundant variables 
they tend to underselect some individually "weak" informative features and select mostly "strong" variables, which would bring no new information with respect to each other. In contrast, our approach searches features that will be \textit{jointly} strong, therefore, it is less prone to missing important variables.

\subsection{Complexity Analysis}

Time complexity of the feature selection methods are compared. The complexity of \texttt{SSLS} and \texttt{SFS} is $O(d n^2)$. Both approaches are very fast on small and medium sample size, however, due to quadratic relation to $n$, they progressively slow down with increase of $n$. \texttt{RLSR} has the complexity $O(\max\{d n^2, d^3\})$ and suffers when either $n$ or $d$ are large. 

\texttt{RF}, used as the base classifier for pseudo-labeling and as the approximation of the Bayes classifier, has linear complexity with respect to dimension and log-linear to the number of observations $O(d n\log n)$, which indicates its good scalability. Nevertheless, the \texttt{FSS} ranking features sequentially has the complexity $O(d^3 n log(n))$, and its combination with co-training limits the use of \texttt{CoT-FSS} in practice.

The time complexity of our wrapper approach with genetic algorithm is $O(\tilde{d} N_g N_c n\log n)$, where $N_g$ is the number of generations, $N_c$ is the population size, $\tilde{d}$ is the maximal number of features used in a learning model. From our observations, $\tilde{d}\approx d/2$ for \texttt{CGA} and $\tilde{d}\approx \sqrt{d}$ for \texttt{FSGA}.

When the sample-size and the dimension are moderate, our approach is computationally more expensive than embedding and filter selections. However, as $d$ grows, \texttt{RLSR} becomes significantly slower. Moreover, the log-linear relation to $n$ indicates that \texttt{SEW-IL} passes the scale and would be faster than the filter methods on very-large data sets, which is usually the case of visual applications.

\section{Conclusion}
\label{sec:conclusion}
In this paper we proposed a new framework for semi-supervised wrapper feature selection.
To increase the diversity of labeled data, unlabeled examples are pseudo-labeled using a self-learning algorithm. We extended the $\mathcal{C}$-bound to the case where these examples are given imperfect class labels. An objective of the proposed wrapper is to minimize this bound using a genetic algorithm. To produce a sparse solution, we proposed a modification of the latter by taking into account feature weights during its evolutionary process. We provided empirical evidence of our framework in comparison with two semi-supervised filter techniques, an embedded approach as well as a wrapper feature selection algorithm. The proposed modification of the genetic algorithm provides a trade-off in tasks, where both high performance and low dimension are reached. 


\bibliographystyle{apalike}
\bibliography{icml-paper.bib}
\end{document}